\documentclass{article}

\usepackage[icml2018]{shortcut}
\def\Bibfile{../../library}
\graphicspath{{../../figures/}}
\def\TikzLocation{../../tikz/}

\usepackage{xr}
\def\crossref{\externaldocument{/home/tom/.cache/gummi/.0_main_icml2018.tex}}
\def\biblio{
	\bibliographystyle{\Bibstyle}
	\bibliography{\Bibfile}
	\def\biblio{}
}

\icmltitlerunning{DICOD: Distributed Convolutional Sparse Coding}

\begin{document}

\def\biblio{}
\def\CHANGES{}
\def\crossref{}

\twocolumn[
\icmltitle{DICOD: Distributed Convolutional Coordinate Descent for\\Convolutional  Sparse Coding}

% It is OKAY to include author information, even for blind
% submissions: the style file will automatically remove it for you
% unless you've provided the [accepted] option to the icml2018
% package.

% List of affiliations: The first argument should be a (short)
% identifier you will use later to specify author affiliations
% Academic affiliations should list Department, University, City, Region, Country
% Industry affiliations should list Company, City, Region, Country

% You can specify symbols, otherwise they are numbered in order.
% Ideally, you should not use this facility. Affiliations will be numbered
% in order of appearance and this is the preferred way.
\icmlsetsymbol{equal}{*}

\begin{icmlauthorlist}
\icmlauthor{Moreau Thomas}{cmla}
\icmlauthor{Oudre Laurent}{P13}
\icmlauthor{Vayatis Nicolas}{cmla}
\end{icmlauthorlist}

\icmlaffiliation{cmla}{CMLA, ENS Paris-Saclay, Université Paris-Saclay, Cachan, France}
\icmlaffiliation{P13}{L2TI, Université Paris 13, Villetaneuse, France}

\icmlcorrespondingauthor{Moreau Thomas}{thomas.moreau@cmla.ens-cachan.fr}

% You may provide any keywords that you
% find helpful for describing your paper; these are used to populate
% the "keywords" metadata in the PDF but will not be shown in the document
\icmlkeywords{Machine Learning, ICML}

\vskip 0.3in
]

% this must go after the closing bracket ] following \twocolumn[ ...

% This command actually creates the footnote in the first column
% listing the affiliations and the copyright notice.
% The command takes one argument, which is text to display at the start of the footnote.
% The \icmlEqualContribution command is standard text for equal contribution.
% Remove it (just {}) if you do not need this facility.

\printAffiliationsAndNotice{}  % leave blank if no need to mention equal contribution

\crossref{}
\begin{abstract}

In this paper, we introduce DICOD, a convolutional sparse coding algorithm which builds shift invariant representations for long signals. This algorithm is designed to run in a distributed setting, with local message passing, making it communication efficient. It is based on coordinate descent and uses locally greedy updates which accelerate the resolution compared to greedy coordinate selection. We prove the convergence of this algorithm and highlight its computational speed-up which is super-linear in the number of cores used. We also provide empirical evidence for the acceleration properties of our algorithm compared to state-of-the-art methods.

\end{abstract} 

\section{Convolutional Representation for Long Signals}
\label{sec:dicod:context}

Sparse coding aims at building sparse linear representations of a data set based on a dictionary of basic elements called atoms. It has proven to be useful in many applications, ranging from EEG analysis to images and audio processing \citep{Adler2013, Kavukcuoglu2013, Mairal2010, Grosse2007}. Convolutional sparse coding is a specialization of this approach, focused on building sparse, shift-invariant representations of signals. Such representations present a major interest for applications like segmentation or classification as they separate the shape and the localization of patterns in a signal. \CHANGES{This is typically the case for physiological signals which can be composed of recurrent patterns linked to specific behavior in the human body such as the characteristic heartbeat pattern in ECG recordings.} Depending on the context, the dictionary can either be fixed analytically (\emph{e.g.} wavelets, see \citealt{Mallat2008}), or learned from the data \citep{Bristow2013, Mairal2010}.

Several algorithms have been proposed to solve the convolutional sparse coding. The Fast Iterative Soft-Thresholding Algorithm (FISTA) was adapted for convolutional problems in \citet{Chalasani2013} and uses proximal gradient descent to compute the representation. The Feature Sign Search (FSS), introduced in \citet{Grosse2007}, solves at each step a quadratic subproblem for an active set of the estimated nonzero coefficients and the Fast Convolutional Sparse Coding (FCSC) of \citet{Bristow2013} is based on Alternating Direction Method of Multipliers (ADMM). Finally, the coordinate descent (CD) has been extended by \citet{Kavukcuoglu2013} to solve the convolutional sparse coding. This method greedily optimizes one coordinate at each iteration using fast local updates. We refer the reader to \citet{Wohlberg2016} for a detailed presentation of these algorithms.

To our knowledge, there is no scalable version of these algorithms for long signals. This is a typical situation, for instance, in physiological signal processing where sensor information can be collected for a few hours with sampling frequencies ranging from 100 to $1000$Hz.
\CHANGES{The existing algorithms for generic $\ell_1$-regularized optimization can be accelerated by improving the computational complexity of each iteration.} A first approach to improve the complexity of these algorithms is to estimate the non-zero coefficients of the optimal solution to reduce the dimension of the optimization space, using either screening \citep{Ghaoui2012,Fercoq2015} or active-set algorithms \citep{Johnson2015}. Another possibility is to develop parallel algorithms which compute multiple updates simultaneously. Recent studies have considered distributing coordinate descent algorithms for general $\ell_1$-regularized minimization \citep{Scherrer2012a, Scherrer2012, Bradley2011, Yu2012}. These papers propose synchronous algorithms using either locks or synchronizing steps to ensure the convergence in general cases. \citet{You2016} derive an asynchronous distributed algorithm for the projected coordinate descent which uses centralized communication and finely tuned step size to ensure the convergence of their method.

% As this research direction is orthogonal to support estimation methods, it is possible to use them jointly with our algorithm. The evaluation of the performances of our algorithm with active-set or screening strategies is left for future work.

% Contribution
In the present paper, we design a novel distributed algorithm tailored for the convolutional problem which is based on coordinate descent, named Distributed Convolution Coordinate Descent (DICOD). DICOD is asynchronous and each process can run independently without locks or synchronization steps. This algorithm uses a local communication scheme to reduce the number messages between the processes and does not rely on external learning rates. We also prove that this algorithm scales super-linearly with the number of cores compared to the sequential CD, up to certain limitations.

% Plan
In \autoref{sec:DCP}, we introduce the DICOD algorithm for the resolution of convolutional sparse coding. Then, we prove in \autoref{sec:analysis} that DICOD converges to the optimal solution for a wide range of settings and we analyze its complexity. Finally, \autoref{sec:result} presents numerical experiments that illustrate the benefits of the DICOD algorithm with respect to other state-of-the-art algorithms and validate our theoretical analysis.

% section dicod:context (end)

\crossref{}

%%%%%% Algorithm

\begin{figure*}[t]
	 \begin{minipage}[t]{.49\textwidth}
	 	\begin{algorithm}[H]
			\begin{algorithmic}[1]

\STATE \textbf{Input: }$\pmb D, X$, parameter $\epsilon >  0$
\STATE $\mathcal C = \llbracket1, K\rrbracket\times\llbracket0, L-1\rrbracket$ 
\STATE Initialization: $\forall (k,t) \in \mathcal C,~~ $  \\
 $Z_k[t] = 0,~~\beta_k[t] = \left(\widetilde{\pmb D_k} * X\right)[t]$
\REPEAT
\STATE $\forall(k,t) \in \mathcal C,~\displaystyle Z'_k[t] = \frac{1}{\|\pmb D_k\|_2^2}\text{Sh}(\beta_k[t], \lambda)~,$
\STATE Choose $\displaystyle(k_0, t_0) = \arg\max_{(k, t)\in\mathcal C} |\Delta Z_k[t]|$
\STATE Update $\beta$ using \autoref{eq:beta_up} and $Z_{k_0}[t_0] \leftarrow Z'_{k_0}[t_0]$
\UNTIL{$|\Delta Z_{k_0}[t_0]| < \epsilon$}

			\end{algorithmic}
			\caption{Greedy Coordinate Descent} 
			\label{alg:cd} 
		\end{algorithm}
	 \end{minipage}%
	\hfill
	 \begin{minipage}[t]{.5\textwidth}
		\begin{algorithm}[H]
		\begin{algorithmic}[1]

\STATE \textbf{Input: }$\pmb D, X$, parameter $\epsilon >  0$
\STATE \textbf{In parallel } for $m=1\cdots M$
\STATE For all $(k,t)$ in $\mathcal C_m$, initialize $\beta_k[t]$ and $Z_k[t]$% 
\REPEAT
\STATE Receive messages and update $\beta$ with \autoref{eq:beta_up}% 
\STATE  $\forall (k, t) \in \mathcal C_m$, compute $Z'_k[t]$ with \autoref{eq:pb_coord} 
\STATE Choose $\displaystyle(k_0, t_0) = \arg\max_{(k, t)\in\mathcal C_m} |\Delta Z_k[t]|$
\STATE Update $\beta$ with \autoref{eq:beta_up} and $Z_{k_0}[t_0]\leftarrow{}Z'_{k_0}[t_0]$%
\STATE \textbf{if}{ $t_0 - mL_M < W$ \textbf{then}}
\STATE ~~~~Send $(k_0, t_0, \Delta Z_{k_0}[t_0])$ to core $m-1$ 
\STATE \textbf{if}{ $(m+1)L_M - t_0< W$ \textbf{then}}
\STATE ~~~~Send $(k_0, t_0, \Delta Z_{k_0}[t_0])$ to core $m+1$ 
\UNTIL{for all cores, $|\Delta Z_{k_0}[t_0]| < \epsilon$ }
		
		\end{algorithmic}
		\caption{DICOD$_M$} 
		\label{alg:dicod} 
		\end{algorithm}
	 \end{minipage}
\end{figure*}

\section{Distributed Convolutional Coordinate Descent (DICOD)}
\label{sec:DCP}

\paragraph{Notations.}
\label{par:contrib:notation}

The space of multivariate signals of length $T$ in $\Rset^P$ is denoted by $\Xset_T^P$~. For these signals, their value at time $t \in \llbracket 0, T-1\rrbracket$ is denoted by $ X[t] \in \Rset^{P}$ and for all $t \notin\llbracket 0, T-1\rrbracket$, $X[t] = \pmb 0_P$. The indicator function of $t_0$ is denoted $\pmb 1_{t_0}$. For any signal $X \in \Xset_T^P$, the reversed signal is defined as $\widetilde X[t] = X[T-t]$, the d-norm is defined as \mbox{$\|X\|_d = \left(\sum_{t=0}^{T-1}\|X[t]\|_d^d\right)^{1/_d}$} and the replacement operator as $\Phi_{t_0}(X)[t] = (1 - \pmb 1_{t_0}(t))X[t]~,$ which replaces the value at time $t_0$ by $0$.  Finally, for $L, W \in \mathbb N^*$, the convolution between $Z \in \Xset_L^1$ and $\pmb D \in \Xset_W^P$ is a multivariate signal
$Z\ast \pmb D \in \Xset_T^P$ with $\smeq T = L + W - 1$ such that for $t \in \llbracket 0, T-1\rrbracket$,
\[
	(Z \ast \pmb D)[t] \overset{\Delta}{=}  \sum_{\tau=0}^{W-1}Z[t-\tau] \pmb D[\tau]~.
\]

% paragraph contrib:notation (end)

This section reviews in \autoref{sub:problem} the convolutional sparse coding as an $\ell_1$-regularized optimization problem and the coordinate descent algorithm to solve it. \CHANGES{Then, \autoref{sub:dicod} and \autoref{sub:dicod:SeqDICOD} respectively introduce the Distributed Convolutional Coordinate Descent (DICOD)  and the Sequential DICOD (SeqDICOD) algorithms to efficiently solve convolutional sparse coding for long signals. Finally, \autoref{sub:related_distrib} discusses related work on $\ell_1$-regularized coordinate descent algorithms.
}

\subsection{Coordinate Descent for Convolutional Sparse Coding}
\label{sub:problem}

\paragraph{Convolutional Sparse Coding.}
\label{par:contrib:csc}

Consider the multivariate signal $X \in \Xset_T^P$. Let $\pmb D = \left \{\pmb D_k \right \}_{k=1}^K \subset \Xset_W^P$ be a set of $K$ patterns with  $W \ll T$ and $Z = \{Z_k\}_{k=1}^K \subset \Xset_L^1$ be a set of $K$ activation signals with $\smeq L = T-W+1$. The convolutional sparse representation models a multivariate signal $X$ as the sum of $K$ convolutions between a local pattern $\pmb D_k$ and an activation signal $Z_k$ such that:
\begin{equation}\label{eq:model}
	X[t] =  \sum_{k=1}^K (Z_k* \pmb D_{k})[t] + \mathcal{E}[t], ~~~~~~~\forall t\in \llbracket 0, T-1\rrbracket,
\end{equation}
with $\mathcal{E} \in \Xset_T^P$ representing an additive noise term. This model also assumes that the coding signals $Z_k$ are sparse, in the sense that only few entries are nonzero in each signal. The sparsity property forces the representation to display localized patterns in the signal. Note that this model can be extended to higher order signals such as images by using the proper convolution operator. In this study, we focus on 1D-convolution for the sake of simplicity.

Given a dictionary of patterns $\pmb D$, convolutional sparse coding aims to retrieve the sparse decomposition $Z^*$ associated to the signal $X$ by solving the following $\ell_1$-regularized optimization problem
\begin{align}
	\label{eq:sparse_code}
	Z^* = & \argmin_{Z=(Z_1, \ldots Z_K)} E(Z)~,~~~~~~~~~~~~~\text{where} \\
	E(Z) \overset{\Delta}{=} & \frac{1}{2} \left\|X - \sum_{k=1}^KZ_k* \pmb D_k\right\|_2^2 + \lambda\sum_{k=1}^K\left\|Z_k\right\|_1,
\end{align}
for a given regularization parameter $\lambda > 0$~. The problem formulation \autoref{eq:sparse_code} can be interpreted as a special case of the LASSO problem with a band circulant matrix. Therefore, classical optimization techniques designed for LASSO can easily be applied to solve it with the same convergence guarantees. \citet{Kavukcuoglu2013} adapted the coordinate descent to efficiently solve the convolutional sparse coding.

% paragraph contrib:csc (end)

\paragraph{Convolutional Coordinate Descent.}
\label{par:contrib:cd}

The coordinate descent is a method which updates one coordinate at each iteration. This type of optimization algorithms is efficient for sparse optimization problem since few coefficients need to be updated to find the optimal solution and the greedy selection of updated coordinates is a good strategy to achieve fast convergence to the optimal point. \mref{alg:cd}{Algorithm} summarizes the greedy convolutional coordinate descent.
% The localized updates make it natural to consider the parallelization of such algorithm.

The method proposed by \citet{Kavukcuoglu2013} iteratively updates at each iteration one coordinate $(k_0, t_0)$ of the coding signal $Z$ to its optimal value $Z'_{k_0}[t_0]$ when all other coordinates are fixed. A closed form solution exists to compute the value $Z'_{k_0}[t_0]$ for the update,
\begin{equation} \label{eq:pb_coord}
	Z'_{k_0}[t_0]
		= \frac{1}{\|\pmb D_{k_0}\|_2^2}\text{Sh}(\beta_{k_0}[t_0], \lambda),
\end{equation}
with the soft thresholding operator defined as
\[
	\smeq\text{Sh}(u, \lambda) = \text{sign}(u)\max(|u| - \lambda, 0).
\]
and an auxiliary variable $\beta \in \Xset^K_L$ defined as
\[
	\smeq\displaystyle \beta_{k}[t] = \left(\widetilde{\pmb D_{k}}*\left(X-
		\sum_{\substack{k'=1\\k'\neq k}}^KZ_{k'}*\pmb D_{k'} - \Phi_{t}\left(Z_{k}\right)*\pmb D_{k}\right)\right)[t]~,
\]
Note that $\beta_k[t]$ is simply the residual when $Z_k[t]$ is equal to 0.

The success of this algorithm highly depends on the efficiency in computing this coordinate update. For problem \autoref{eq:sparse_code}, \citet{Kavukcuoglu2013} show that if at iteration $q$, the coefficient $(k_0, t_0)$ of $Z^{(q)}$ is updated to the value $Z'_{k_0}[t_0]$, then it is possible to compute $\beta^{(q+1)}$ from $\beta^{(q)}$ using
\begin{equation}\label{eq:beta_up}
	\beta_k^{(q+1)}[t] = \beta_k^{(q)}[t] -
		\mathcal S_{k, k_0}[t-t_0] \Delta Z^{(q)}_{k_0}[t_0],
\end{equation}
for all $(k, t) \neq (k_0, t_0)$ with \mbox{$\mathcal S_{k, l}[t] = (\widetilde{\pmb D_k} * \pmb D_l)[t]$}~. For all \mbox{$t \notin\llbracket-W+1, W-1\rrbracket$}, $\mathcal S[t]$ is zero. Thus, only \mbox{$\mathcal O(KW)$} operations are needed to maintain $\beta$ up-to-date with the current estimate $Z$. In the following,
\[
	\Delta E_{k_0}[t_0] = E(Z^{(q)}) - E(Z^{(q+1)})
\]
denotes the cost variation obtained when the coefficient $(k_0, t_0)$ is replaced by its optimal value $Z'_{k_0}[t_0]$.

The selection of the updated coordinate $(k_0, t_0)$ can follow different strategies. Cyclic updates \citep{Friedman2007} and random updates \citep{Shalev2009} are efficient strategies as they have a $\bO{1}$ computational complexity. \citet{Osher2009} propose to select the coordinate greedily to maximize the cost reduction of the update. In this case, the coordinate is chosen as the one with the largest difference $\max_{(k, t)}\left|\Delta Z_{k}[t]\right|$ between its current value $Z_k[t]$ and the value $Z'_{k}[t]$ with
\begin{equation}
	\Delta Z_k[t] = Z_k[t] - Z'_k[t]
\end{equation}
This strategy is computationally more expensive, with a cost of $\bO{KT}$ but it has a better convergence rate \citep{Nutini}. In this paper, we focus on the greedy approach as it aims to get the largest gain from each update. Moreover, as the updates in the greedy scheme are more complex to compute, distributing them provides a larger speedup compare to other strategies.

The procedure is run until $\max_{k, t}|\Delta Z_{k}[t]|$ becomes smaller than a specified tolerance parameter $\epsilon$.

% paragraph contrib:cd (end)

% subsection problem (end)
%%%%%%%%%%%%%%%%%%%%%%%%%%%%%%%%%%%%%%%%%%%%%%%%%%%%%%%%%%%%%%%%%%%%%%%%%%%%%%%%%%%%%%%%%%%%%%%%%%%%%%%%%%%%%%%%%%%%%%%%%%%%%%%%%%%%%%%%%%%%%%%%%%%%%%%%%%%%%%%%%%%%%%%%%%%%%%%%%%%%%%%%%%%%%%%%%%%%%%%%%%%%%%%%%%%%%%%%%%%%%%%%%%%%%%

\subsection{Distributed Convolutional Coordinate Descent (DICOD)}
\label{sub:dicod}

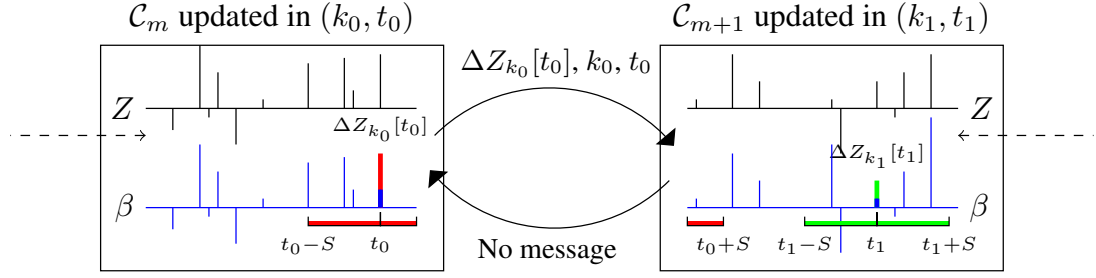
\begin{figure*}[tp]
\centering
\scalebox{1.2}{
%\usetikzlibrary{...}

%\newcommand{\smeq}{\medmuskip=0mu \thickmuskip=0mu \thinmuskip=0mu}

\begin{tikzpicture}
\pgfmathsetmacro{\iO}{-1.9}
\pgfmathsetmacro{\ione}{3.6}
\pgfmathsetmacro{\heighl}{-0.8}
\pgfmathsetmacro{\heighz}{0.3}
\pgfmathsetmacro{\sep}{0.05}
\pgfmathsetmacro{\S}{0.8}
\pgfmathsetmacro{\r}{(-\iO-1.5)/\S}

\draw (-5, 1) -- (-1.2, 1) node[midway, above]{$\mathcal C_m$  updated in $(k_0, t_0)$} -- (-1.2, -1.5) -- (-5, -1.5) -- cycle;
\draw (5, 1) -- (1.2, 1) node[midway, above]{$\mathcal C_{m+1}$ updated in $(k_1, t_1)$} -- (1.2, -1.5) -- (5, -1.5) -- cycle;
\draw [<-, dashed] (-4.5, 0) -- (-6, 0);
\draw [<-, dashed] (4.5, 0) -- (6, 0);
\foreach \s in {-1, 1}
 \draw[blue] (\s*4.5, \heighl) -- (\s*1.5, \heighl) ;
\draw[green, line width = 1.5pt] (\ione, \heighl) -- ++(0, 0.3) node[above, black] {\tiny$\Delta Z_{k_1}[t_1]$};
\draw[red, line width = 1.5pt] (\iO, \heighl) -- ++(0, 0.6)  node[above, black] {\tiny$\Delta Z_{k_0}[t_0]$};
\foreach \x/\y in {2.3/0.3, 1.6/0.1, 2/0.6, 3.1/0.7, 4.2/1, 3.9/0.4, 3.8/-0.1, 3.2/-0.5, \ione/0.1}
    \draw[blue] (\x, \heighl) -- ++(0, \y) ;
\draw[blue, line width = 1.5pt] (\ione, \heighl) -- ++(0, .1) ;
\foreach \x/\y in {2.3/0.56, 2.2/0.2, 2.7/0.5, 3.7/0.4, 4.2/-0.24, 3.9/0.7, 3.8/-0.1, 3.2/0.1, 3.5/-0.4, -\iO/0.2}
    \draw[blue] (-\x, \heighl) -- ++(0, \y);
\draw[blue, line width = 1.5pt] (\iO, \heighl) -- ++(0, .2) ;
\fill[red] (\iO-\S, \heighl-3*\sep) -- ++(0, -\sep) -- ++(\S+\r*\S, 0) -- ++(0, \sep) -- cycle;
\draw (\iO+\r*\S, \heighl-2*\sep) -- ++(0, -2*\sep) -- ++ (-\r*\S, 0)  -- ++ (0, 0.15) -- ++ (0, -0.15) node [below] {\tiny $t_0$} -- ++(-0.8, 0) node[below, black] {\tiny \smeq$t_0 - S$} -- ++(0, +0.1);

\fill[red] (1.5, \heighl-3*\sep) -- ++(0, -\sep) -- ++(\S-\r*\S, 0) -- ++(0, \sep) -- cycle;
\draw (1.5, \heighl-2*\sep) -- ++(0, -0.1) -- ++(\S-\r*\S, 0) node[below, black] {\tiny  \smeq$t_0 + S$} -- ++(0, +0.1);
\fill[green] (\ione-\S, \heighl-3*\sep) -- ++(0, -\sep) -- ++(2*\S, 0) -- ++(0, \sep) -- cycle;
\draw (\ione-\S, \heighl-2*\sep) -- ++(0, -2*\sep) node [below] {\tiny \smeq$t_1-S$} -- ++(\S, 0) -- ++(0, 3*\sep)  -- ++(0, -3*\sep) node [below] {\tiny $t_1$}-- ++(\S, 0) node [below] { \smeq \tiny$t_1+S$} -- ++(0, 2*\sep);
%\draw (-3.5, 1) node[below] {};
\draw (-4.5, \heighl) node[left] {$\beta$};
\draw (4.5, \heighl) node[right] {$\beta$};

%% Z on top in red
 
\draw[black] (4.5, \heighz) -- (1.5, \heighz) ;
\foreach \x/\y in {2.3/0.3, 1.6/0.1, 2/0.6, 3.1/0.1, 4.2/0.6, 3.9/0.4, 3.8/-0.1, 3.2/-0.5, \ione/0.3}
    \draw[black] (\x, \heighz) -- ++(0, \y) ;
\draw[black] (-4.5, \heighz) -- (-1.5, \heighz) ;
\foreach \x/\y in {2.3/0.56, 2.2/0.2, 2.7/0.5, 3.7/0.4, 4.2/-0.24, 3.9/0.7, 3.8/-0.1, 3.2/0.1, 3.5/-0.4, -\iO/0.6}
    \draw[black] (-\x, \heighz) -- ++(0, \y);
\draw (-4.5, \heighz) node[left] {$Z$};
\draw (4.5, \heighz) node[right] {$Z$};

\draw[-triangle 90] (-1.3, 0) arc (135:40:1.8) node[midway, above] {\small$\smeq\Delta Z_{k_0}[t_0]$, $k_0$, $t_0$};% -- ++ (0.1, -0.1);
\draw[-triangle 90] (1.3, -0.5) arc (135:40:-1.8) node[midway, below] {\small No message};% -- ++ (-0.1, 0.1);
\end{tikzpicture}
}
\caption{
	Communication process in DICOD for two cores $\mathcal C_m$ and $\mathcal C_{m+1}$.
	(\emph{red}) The process needs to send a message to its neighbor as it updates a coefficient
	with $t_0$ located near the border of the core's segment, in the interference zone.
	(\emph{green}) The update in $t_1$ is independent of other cores.
	% Core $m+1$ updates $(k_1, t_1) \in \mathcal C_{m+1}$ independently from the other
    % cores as it is located out of the interference zones.
	% Core $m$ is updating a coordinate $(k_0, t_0) \in \mathcal C_{m}$ which is
	% in the interference zone $\smeq\llbracket mL_M-S, mL_M\rrbracket$. Therefore, it needs
	% to notify the process $m+1$ of the update by sending a message
	% composed of the value of the update $\Delta Z_{k_0}[t_0]$ and its location $(k_0, t_0)$.
	% When core $m+1$ retrieves the pending message from $m$, it
	% will update $\beta$ to take into account the update of $Z_{k_0}[t_0]$.
}
\label{fig:DCP}
\end{figure*}

For convolutional sparse coding, the coordinate descent updates are only weakly dependent as it is shown in \autoref{eq:beta_up}. It is thus natural to parallelize it for this problem.

\paragraph{DICOD.}

\mref{alg:dicod}{Algorithm} describes the steps of DICOD with $M$ workers. Each worker $m \in \llbracket1, M\rrbracket$ is in charge of updating the coefficients  of a segment $\mathcal C_m$ of length \mbox{$L_M = L / M$} defined by:
\[
	\smeq \mathcal C_m = 
	\left \{ (k, t) ~;~~
				k \in \llbracket1, K\rrbracket, ~~
				t\in \left\llbracket(m-1)L_M, ~~mL_M -1 \right\rrbracket \right \}~.
\]
The local updates are performed in parallel for all the cores using the greedy coordinate descent introduced in \autoref{sub:problem}. When a core $m$ updates the coordinate $(k_0, t_0)$ such that $t_0 \in \llbracket (m-1)L_M +W, mL_M -W\rrbracket$, the updated coefficients of $\beta$ are all contained in $\mathcal C_m$ and there is no need to update $\beta$ on the other cores. In these cases, the update is equivalent to a sequential update. When \mbox{$\smeq t_0 \in \llbracket mL_M -W, mL_M \rrbracket$} \mbox{$\displaystyle\smeq \left(\text{resp. } t_0 \in  \llbracket(m-1)L_M, (m-1)L_M+W\rrbracket\right)$,} some of the coefficients of $\beta$ in core $m+1$ (resp. $m-1$) need to be updated and the update is not local anymore. This can be done by sending the position of updated coordinate $(k_0, t_0)$, and the value of the update $\Delta Z_{k_0}[t_0]$ to the neighboring core. \autoref{fig:DCP} illustrates this communication process. Inter-processes communications are very limited in DICOD. One node communicates with its neighbors only when it updates coefficients close to the extremity of its segment. When the size of the segment is reasonably large compared to the size of the patterns, only a small part of the iterations needs to send messages. We cannot apply the stopping criterion of CD in each worker of DICOD, as this criterion might not be reached globally. The updates in the neighbor cores can break this criterion. To avoid this issue, the convergence is considered to be reached once all the cores achieve this criterion simultaneously. Workers that reach this state locally are paused, waiting for incoming communication or for the global convergence to be reached.

The key point that allows distributing the convolutional coordinate descent algorithm is that the solutions on time segments that are not overlapping are only weakly dependent. Equation \autoref{eq:beta_up} shows that a local change has impact on a segment of length $2W-1$ centered around the updated coordinate. Thus, if two coordinates which are far enough were updated simultaneously, the resulting point $Z$ is the same as if these two coordinates had been updated sequentially. By splitting the signal into continuous segments over multiple cores, coordinates can be updated independently on each core up to certain limits.

\paragraph{Interferences.}
\label{par:interferences}

When two coefficients $(k_0, t_0)$ and $(k_1, t_1)$ are updated by two neighboring cores simultaneously, the updates might not be independent and cannot be considered to be sequential. The local version of $\beta$ used for the second update does not account for the first update. We say that the updates are \textit{interfering}. The cost reduction resulting from these two updates is denoted $\Delta E_{k_0, k_1}[t_0, t_1]$ and simple computations, detailed in \autoref{prop:ii}, show that

\begin{equation}
\begin{split}
	\Delta E_{k_0, k_1}[t_0, t_1] &= 
	\overbrace{\Delta E_{k_0}[t_0] + \Delta E_{k_1}[t_1]}^{
		\text{iterative steps}} \\
	 & - \underbrace{\mathcal S_{k_0, k_1}[t_1 - t_0]
		\Delta Z_{k_0}[t_0] \Delta Z_{k_1}[t_1]}_{
			\text{interference}}, \label{eq:interf}
\end{split}
\end{equation}

If $|t_1 - t_0| \ge W$, then $\mathcal S_{k_0, k_1}[t_1-t_0] = 0$ and the updates can be considered to be sequential as the interference term is zero. When $|t_1 - t_0| < W$, the interference term does not vanish  but \autoref{sec:analysis} shows that under mild assumption, this term can be controlled and it does not make the algorithm diverge.

%%%%%% Seq DICOD
\subsection{Randomized Locally Greedy Coordinate Descent (SeqDICOD)}
\label{sub:dicod:SeqDICOD}

The theoretical analysis in \autoref{thm:speedup} shows that DICOD provides a super-linear acceleration compared to the greedy coordinate descent. This result is supported with the numerical experiment presented in \autoref{fig:scaling}. The super-linear speed up results from a double acceleration, provided by the parallelization of the updates -- we update $M$ coefficients at each iteration -- and also by the reduction of the complexity of each iteration. Indeed, each core computes greedy updates with linear in complexity on $1/M$-th of the signal.
%Because the updates are only weakly dependent, choosing the coordinate to update using the locally greedy scheme instead of the greedy does not fundamentally change the convergence rate.
This super-linear speed-up means that running DICOD sequentially will still provide a speed-up compared to the greedy coordinate descent algorithm.

\autoref{alg:SeqDICOD} presents SeqDICOD. This algorithm is a sequential version of DICOD. At each step, one segment $\mathcal C_m$ is selected uniformly at random between the $M$ segments. The greedy coordinate descent algorithm is applied locally on this segment. This update is only locally greedy and maximizes
\[
	(k_0, t_0) = \argmax_{(k, t)\in\mathcal C_m} |\Delta Z_k[t]|
\]
This coordinate is then updated to its optimal value $Z'_{k_0}[t_0]$. In this case, there is no interference as the segments are not updated simultaneously.

Note that if $M = T$, this algorithm becomes very close to the randomized coordinate descent. The coordinate is selected greedily only between the $K$ different channels of the signal $Z$ at the selected time. So the selection of $M$ depends on a tradeoff between the randomized coordinate descent and the greedy coordinate descent.

\begin{algorithm}[tp]
\begin{algorithmic}[1]

	\STATE \textbf{Input: }$\pmb D, X$, parameter $\epsilon >  0$, number of segments $M$
	\STATE Initialize $\beta_k[t]$ and $Z_k[t]$ for all $(k,t)$ in $\mathcal C$
	\STATE Initialize $dZ_m = +\infty$ for $m \in \llbracket 1, M\rrbracket$
	\REPEAT
	\STATE Randomly select $m \in \llbracket 1, M\rrbracket$
	\STATE  $\forall (k, t) \in \mathcal C_m$, compute $Z'_k[t]$ with \autoref{eq:pb_coord}
	\STATE Choose $\displaystyle(k_0, t_0) = \argmax_{(k, t)\in\mathcal C_m} |\Delta Z_k[t]|$
	\STATE Update $\beta$ with \autoref{eq:beta_up}
	\STATE Update the current point estimate $Z_{k_0}[t_0]^{(q+1)}\leftarrow{}Z'_{k_0}[t_0]$%
	\STATE Update max updates vector $dZ_m = \left|Z_{k_0}[t_0]^{(q+1)} - Z'_{k_0}[t_0]\right|$%
	\UNTIL{$\|dZ\|_\infty < \epsilon$ and $\|\Delta Z\|_\infty < \epsilon$  }

\end{algorithmic}
\caption{Locally greedy coordinate descent SeqDICOD$_M$}
\label{alg:SeqDICOD}
\end{algorithm}

%Also, the stopping criterion has to be modified to match the one from the coordinate descent. Indeed, it is necessary to wait until all the segments have their next update amplitude below the threshold to declare that SeqDICOD has converged. We define the vector $dZ \in \Rset^M$ with $dZ_m$ the magnitude of the previous update on segment $m\in \llbracket 1, M\rrbracket~.$ When $\|dZ\|_\infty = \max_m\left|dZ_m\right|$ decrease below $\epsilon$, there is a chance that the algorithm converged and we can check the second condition $\|\Delta Z\|_\infty < \epsilon$. This method avoids the costly check of the second condition at each iteration, with a computational cost of $\bO{KT}$, using the first condition which only has complexity $\bO{M}$. The first condition alone does not guarantee convergence as an update in one segment can increase a coefficient on another segment. The second condition ensures that the convergence is reached, on all segments.

\subsection{Discussion}
\label{sub:related_distrib}

This algorithm differs from the existing paradigm to distribute CD \citep{Scherrer2012a, Scherrer2012, Bradley2011, Yu2012, You2016} as it does not rely on centralized communication. Indeed, other parallel coordinate descent algorithms rely on a parameter server, which is an extra worker that holds the current value of $Z$. As the size of the problem and the number of nodes grow, the communication cost can rapidly become an issue with this kind of centralized communication. The natural workload split proposed with DICOD allows for more efficient interactions  between the workers and reduces the need for inter-node communications. Moreover, to prevent the interferences breaking the convergence, existing algorithms rely either on synchronous updates \citep{Bradley2011, Yu2012} or on reduced step size in the updates \cite{You2016, Scherrer2012a}. In both case, they are less efficient than our asynchronous greedy algorithm that can leverage the convolutional structure of the problem to use both large updates and independent processes without external parameters.

\CHANGES{
As seen in the introduction, another way to improve the computational complexity of sparse coding algorithms is to estimate the non-zero coefficients of the optimal solution in order to reduce the dimension of the optimization space. As this research direction is orthogonal to the parallelization of the coordinate descent, it would be possible to combine our algorithm with either screening \citep{Ghaoui2012,Fercoq2015} or active-set methods \citep{Johnson2015}. The evaluation of the performances of our algorithm with these strategies is left for future work.
}

% paragraph related_distrib (end)

\crossref{}

\section{Properties of DICOD}
\label{sec:analysis}

%The properties of the convolutional sparse coding problem
%permit to distribute intuitively the computations over
%different nodes. In the following, we show that our algorithm
%preserves the convergence properties of the coordinate descent.
%Then, we discuss the acceleration that can be expected
%from this asynchronous parallelization.

%The CD for problem \autoref{eq:sparse_code}
%can be distributed without synchronization for a large range of
%convolutional sparse coding problems. The key point is that the
%synchronization is only here to control the interference that
%might result from the asynchronous updates. But such interferences
%do not break the convergence property of the algorithm when the
%dictionary elements are not too correlated.

\paragraph{Convergence of DICOD.}
The magnitude of the interference is related to the value of the cross-correlation between dictionary elements, as shown in \autoref{prop:interf}. Thus, when the interferences have low probability and small magnitude, the distributed algorithm behaves as if the updates were applied sequentially, resulting in a large acceleration compared to the sequential CD algorithm.

\begin{restatable}{proposition}{PropInterf}\label{prop:interf}
For concurrent updates for coefficients $(k_0, t_0)$ and $(k_1, t_1)$ of a sparse code Z, the cost update $\Delta E_{k_0k_1}[t_0,t_1]$ is lower bounded by
\begin{equation}\begin{split}
	\Delta E_{k_0k_1}[t_0,t_1] \ge &  
		\Delta E_{k_0}[t_0] + \Delta E_{k_1}[t_1]\\
		&- 2 \frac{\mathcal S_{k_0, k_1}[t_0-t_1]}
	{\|\pmb D_{k_0}\|_2\|\pmb D_{k_1}\|_2}\sqrt{\Delta E_{k_0}[t_0]
			\Delta E_{k_1}[t_1]}.
\end{split}\end{equation}
\end{restatable}

The proof of this proposition is given in \mref{sub:annex:cvg:interf}{Appendix}. It relies on the $\|\pmb D_k\|_2^2$-strong convexity of \autoref{eq:pb_coord}, which gives $|\Delta Z_k[t]| \le \frac{\sqrt{2 \Delta E_k[t](Z)}}{\| \pmb D_k\|_2}$ for all $Z$. Using this inequality with \autoref{eq:interf} yields the result.\\

This proposition controls the magnitude of the interference using the cost reduction associated to a single update. When the correlations between the different elements of the dictionary are small enough, the interfering update does not increase the cost function. The updates are less efficient but do not worsen the current estimate. Using this control on the interferences, we can prove the convergence of DICOD.
\begin{restatable}{theorem}{convDICOD}\label{thm:conv_dcp}
Consider the following hypotheses,
\begin{hypothesis}\label{hyp:H1}
	For all $\smeq(k_0, t_0), (k_1, t_1)$ such
	that $\smeq t_0 \neq t_1 $, $~~~\left|\frac{\mathcal S_{k_0, k_1}[t_0-t_1]}
	{\|\pmb D_{k_0}\|_2\|\pmb D_{k_1}\|_2}\right| < 1~.$
\end{hypothesis}
\begin{hypothesis}\label{hyp:H2}
	There exists $A \in \mathbb N^*$ such that all cores
	$m \in \llbracket 1, M\rrbracket$ are updated at least once between iteration
	$i$ and $i+A$ if the solution is not locally optimal.
\end{hypothesis}
\begin{hypothesis}\label{hyp:H3}
	The delay in communication between the processes is inferior to the update time.
\end{hypothesis}
Under \autoref{hyp:H1}-\autoref{hyp:H2}-\autoref{hyp:H3}, the DICOD algorithm converges to the optimal solution $Z^*$ of \autoref{eq:sparse_code}.
\end{restatable}

Assumption \autoref{hyp:H1} is satisfied as long as the dictionary elements are not replicated in shifted positions in the dictionary. It ensures that the cost is updated in the right direction at each step. This assumption can be linked to the shifted mutual coherence introduced in \citet{Papyan2016a}.

Hypothesis \autoref{hyp:H2} ensures that all coefficients are updated regularly if they are not already optimal. This analysis is not valid when one of the cores fails. As only one core is responsible for the update of a local segment, if a worker fails, this segment cannot be updated anymore and thus the algorithm will not converge to the optimal solution.

Finally, under \autoref{hyp:H3}, an interference only results from one update on each core. Multiple interferences occur when a core updates multiple coefficients in the border of its segment before receiving the communication from other processes border updates. When $T \gg W$, the probability of multiple interference is low and this hypothesis can be relaxed if the updates are not concentrated on the borders.

\begin{proof}[\textbf{Proof sketch for \autoref{thm:conv_dcp}.}]
The full proof can be found in \mref{sub:annex:cvg:proof}{Appendix}. The main argument in proving the convergence is to show that most of the updates can be considered sequentially and that the remaining updates do not increase the cost of the current point.
By \autoref{hyp:H3}, for a given iteration, a core can interfere with at most one other core. Thus, without loss of generality, we can consider that at each step $q$, the variation of the cost $E$ is either $\Delta E_{k_0}[t_0](Z^{(q)})$ or $\Delta E_{k_0k_1}[t_0,t_1](Z^{(q)}),$ for some \mbox{$(k_0, t_0), (k_1, t_1) \in \llbracket 1, K\rrbracket\times\llbracket 0, T-1\rrbracket$}~. \autoref{prop:interf} and \autoref{hyp:H1} proves that $\Delta E_{k_0k_1}[t_0,t_1](Z^{(q)}) \ge 0$. For a single update $\Delta E_{k_0}[t_0](Z^{(q)})$, the update is equivalent to a sequential update in CD, with the coordinate chosen randomly between the best in each segments. Thus, $\Delta E_{k_0}[t_0](Z^{(q)}) > 0$ and the convergence is eventually proved using results from \citet{Osher2009}.
\end{proof}

\paragraph{Speedup of DICOD.}
\label{par:analysis:speedup}

We denote $S_{cd}(M)$ the speedup of DICOD compared to the sequential greedy CD. This quantifies the number of iterations that can be run by DICOD during one iteration of CD.

\begin{restatable}{theorem}{SpeedDICOD}\label{thm:speedup}
Let $\alpha = \frac{W}{T}$ and $M\in\mathbb N^*~.$ If $\alpha M < \frac{1}{4}$ and if the non-zero coefficients of the sparse code are distributed uniformly in time, the expected speedup $\mathbb E [S_{cd}(M)]$ is lower bounded by
\[
	\mathbb E[S_{cd}(M)] \ge M^2 ( 1- 2\alpha^2M^2\left(\smeq1 +2\alpha^2M^2\right)^{\frac{M}{2}-1})~.
\]
\end{restatable}

This result can be simplified when the interference probability $(\alpha M)^2$  is small.

\begin{restatable}{corollary}{corSpeedDICOD}
\label{corr:speedup}
The expected speedup $\mathbb E [S_{cd}(M)]$ when $M\alpha \to 0$ is such that
\[
	\mathbb E[S_{cd}(M)] \underset{\alpha \to 0}{\gtrsim} M^2(1-2\alpha^2M^2 + \mathcal O(\alpha^4M^4))~.
\]
\end{restatable}

\begin{proof}[\textbf{Proof sketch for \autoref{thm:speedup}.}]
The full proof can be found in \autoref{sec:annex:speedup}. There are two aspects involved in DICOD speedup: the computational complexity and the acceleration due to the parallel updates. As stated in \autoref{sub:problem}, the complexity of each iteration for CD is linear with the length of the input signal $T$. In DICOD, each core runs on a segment of size $\frac{T}{M}$. This accelerates the execution of individual updates by a factor $M$. Moreover, all the cores compute their update simultaneously. The updates without interference are equivalent to sequential updates. Interfering updates happen with probability $\left(M\alpha\right)^2$ and do not increase the cost. Thus, one  iteration of DICOD with $N_i$ interferences provides a cost variation equivalent to $M - 2N_i$ iterations using sequential CD and, in expectation, it is equivalent to $M - 2\mathbb E[N_i]$ iterations of DICOD. The probability of interference depends on the ratio between the length of the segments  used for each core and the size of the dictionary. If all the updates are spread uniformly on each segment, the probability of interference between 2 neighboring cores is $\left(\frac{MW}{T}\right)^2$. The expected number of interference $\mathbb E[N_i]$ can be upper bounded using this probability and this yields the desired result.
\end{proof}

The overall speedup of DICOD is super-linear compared to sequential greedy CD for the regime where $\alpha M \ll 1$. It is almost quadratic for small $M$ but as $M$ grows, there is a sharp transition that significantly deteriorates the acceleration provided by DICOD. \autoref{sec:result} empirically highlights this behavior. For a given $\alpha$, it is possible to approximate the optimal number of cores $M$ to solve convolutional sparse coding problems.

Note that this super-linear speed up is due to the fact that CD is inefficient for long signals, as its iterations are computationally too expensive to be competitive with the other methods. The fact that we have a super-linear speed-up means that running DICOD sequentially will provide an acceleration compared to CD (see \autoref{sub:dicod:SeqDICOD}). For the sequential run of DICOD, called SeqDICOD, we have a linear speed-up in comparison to CD, when $M$ is small enough. Indeed, the iteration cost is divided by $M$ as we only need to find the maximal update on a local segment of size $\frac{T}{M}$. When increasing $M$ over $\frac{T}{W}$, the iteration cost does not decrease anymore as updating $\beta$ costs $\bO{KW}$ and finding the best coordinate has the same complexity.

\crossref{}

\section{Numerical Results}
\label{sec:result}

\begin{figure}[tp]
	\centering
	\includegraphics[width=.49\textwidth]{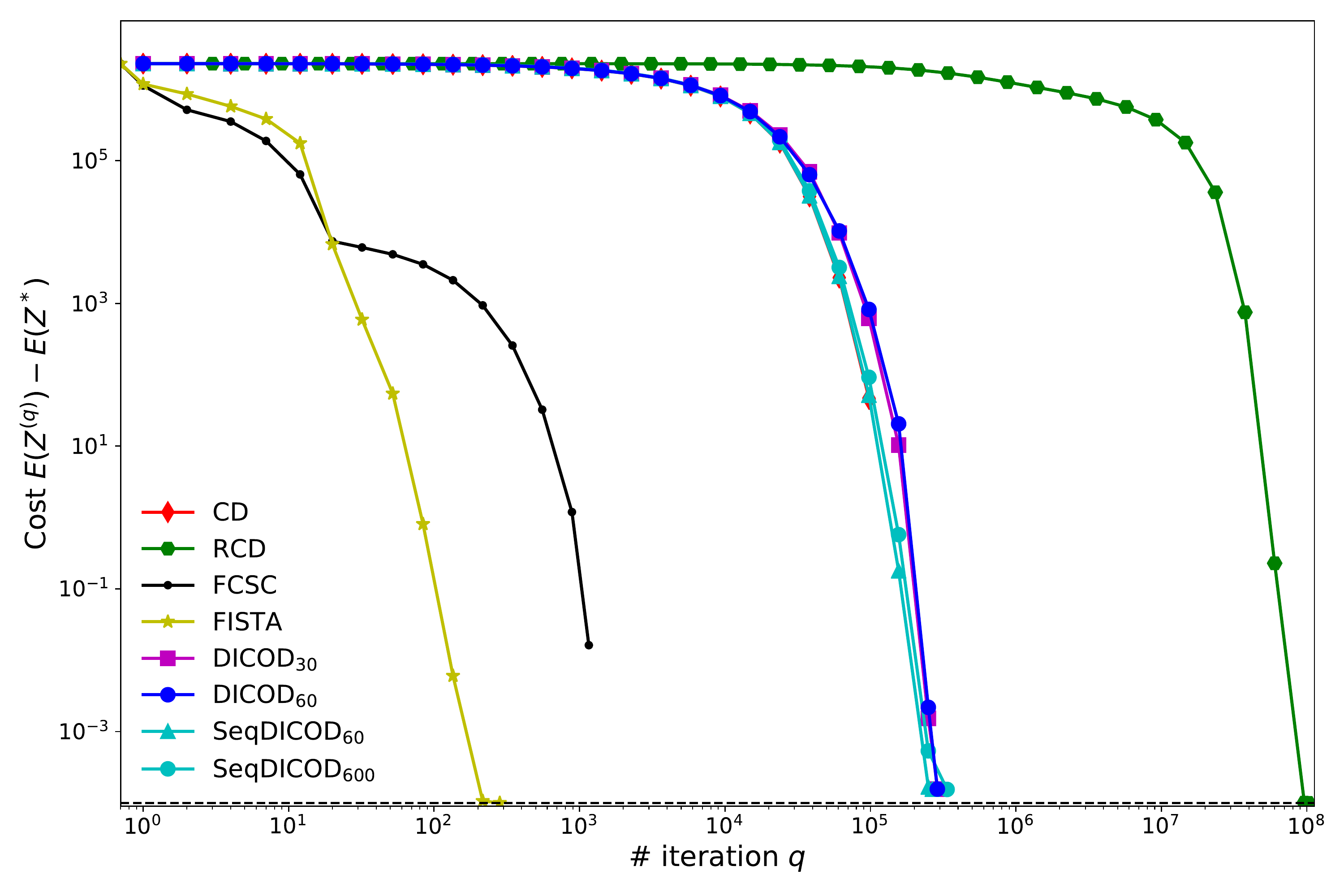}
	\caption{
		Evolution of the loss function for DICOD, SeqDICOD,
		CD, FCSC and FISTA while solving sparse coding for a
		signal generated with default parameters relatively
		to the number of iterations.
	}
	\label{fig:cost:iter}
\end{figure}

All the numerical experiments are run on five Linux machines with 16 to 24 Intel Xeon 2.70 GHz processors and at least 64 GB of RAM on local network. We use a combination of Python, C++ and the OpenMPI 1.6 for the algorithms implementation. The code to reproduce the figures is available online \footnote{\href{https://github.com/}{see the supplementary materials}}. The run time denotes the time for the system to run the full algorithm pipeline, from cold start and includes for instance the time to start the sub-processes. The convergence refers to the variation of the cost with the number of iterations and the speed to the variation of the cost relative to time.

\paragraph{Long convolutional Sparse Coding Signals.}
\label{par:artif}

To further validate our algorithm, we generate signals and test the performances of DICOD compared to state-of-the-art methods proposed to solve convolutional sparse coding. We generate a signal $X$ of length $T$ in $\Rset^P$ following the model described in \autoref{eq:model}. The $K$ dictionary atoms $\pmb D_k$ of length $W$ are drawn as a generic dictionary. First, each entry is sampled from a Gaussian distribution. Then, each pattern is normalized such that $\|\pmb D_k\|_2 = 1$. The sparse code entries are drawn from a Bernoulli-Gaussian distribution  with Bernoulli parameter $\rho = 0.007$, mean $0$ and standard variation $\sigma=10$ . The noise term $\mathcal E$ is chosen as a Gaussian white noise with variance 1. The default values for the dimensions are set to $W=200$, $K=25$, $P=7$, $T=600\times W$ and we used $\lambda = 1$.

% subsection res (end)
%%%%%%%%%%%%%%%%%%%%%%%%%%%%%%%%%%%%%%%%%%%%%%%%%%%%%%%%%%%%%%%%%%%%%%%%%%%%%%%%%%%%%%%%%%%%%%%%%%%%%%%%%%%%%%%%%%%%%%%%%%%%%%%%%%%%%%%%%%%%%%%%%%%%%%%%%%%%%%%%%%%%%%%%%%%%%%%%%%%%%%%%%%%%%%%%%%%%%%%%%%%%%%%%%%%%%%%%%%%%%%%%%%%%%%

%\subsection{Implementation details}
%\label{sub:impl}

\paragraph{Algorithms Comparison.}

DICOD is compared to the main state-of-the-art optimization algorithms for convolutional sparse coding: Fast Convolutional Sparse Coding (FCSC) from \citet{Bristow2013}, Fast Iterative Soft Thresholding Algorithm (FISTA) using Fourier domain computation as described in \citet{Wohlberg2016}, the greedy convolutional coordinate descent (CD, \citealt{Kavukcuoglu2013}) and the randomized coordinate descent (RCD, \citealt{Nesterov2010}). All the specific parameters for these algorithms are fixed based on the authors' recommendations. DICOD$_M$ denotes the DICOD algorithm run using M cores. We also include SeqDICOD$_{M}$, for $M\in \{60, 600\}$, the sequential run of the DICOD algorithm using $M$ segments, as described in \autoref{alg:SeqDICOD}.

\begin{figure}[tp]
	\centering
	\includegraphics[width=.49\textwidth]{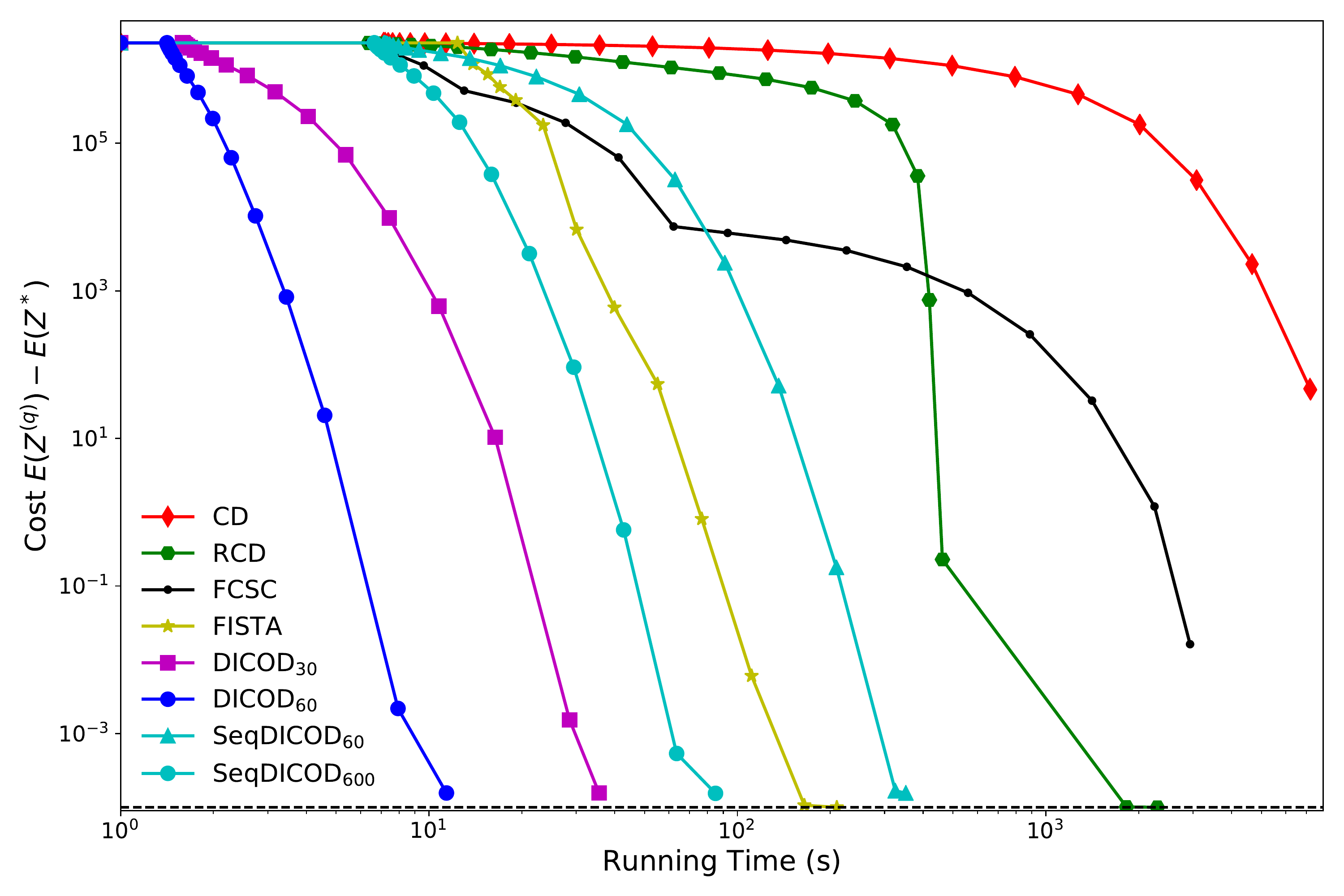}
	\caption{
		Evolution of the loss function for DICOD, SeqDICOD,
		CD, FCSC and FISTA while solving sparse coding for a
		signal generated with default parameters, relatively to
		time. This highlights the speed of the algorithm on the
		given problem.
	}
	\label{fig:cost:time}
\end{figure}

\autoref{fig:cost:iter} shows that the evolution of the performances of SeqDICOD relatively to the iterations are very close to the performances of CD. The difference between these two algorithms is that the updates are only locally greedy in SeqDICOD. As there is little difference visible between the two curves, this means that in this case, the computed updates are essentially the same. The differences are larger for SeqDICOD$_{600}$, as the choice of coordinates are more localized in this case. The performance of DICOD$_{60}$ and DICOD$_{30}$ are also close to the iteration-wise performances of CD and SeqDICOD. The small differences between DICOD and SeqDICOD result from the iterations where there are interferences. Indeed, if two iterations interfere, the cost does not decrease as much as if the iterations were done sequentially. Thus, it requires more steps to reach the same accuracy with DICOD$_{60}$ than with SeqDICOD and with DICOD$_{30}$, as there are more interferences when the number of cores $M$ increases. This explains the discrepancy in the decrease of the cost around the iteration $10^5$. However, the number of extra steps required is quite low compared to the total number of steps and the performances are mostly not affected by the interferences. The performances of RCD in terms of iterations are much slower than the greedy methods. Indeed, as only a few coefficients are useful, it takes many iterations to draw them randomly. In comparison, the greedy methods are focused on the coefficients which largely divert from their optimal value, and are thus most likely to be important. Another observation is that the performance in term of number of iterations of the global methods FCSC and FISTA are much better than the methods based on local updates. As each iteration can update all the coefficients for FISTA, the number of iterations needed to reach the optimal solution is indeed smaller than for CD, where only one coordinate is updated at a time.

In \autoref{fig:cost:time}, the speed of theses algorithms can be observed. Even though it needs many more iterations to converge, the randomized coordinate descent is faster than the greedy coordinate descent. Indeed, for very long signals, the iteration complexity of greedy CD is prohibitive. However, using the locally greedy updates, with SeqDICOD$_{60}$ and SeqDICOD$_{600}$, the greedy algorithm can be made more efficient. SeqDICOD$_{600}$ is also faster than the other state-of-the-art algorithms FISTA and FCSC. The choice of $M = 600$ is a good tradeoff for SeqDICOD as it means that the segments are of the size of the dictionary $W$. With this choice for $M = \frac{T}{W}$, the computational complexity of choosing a coordinate is $\bO{KW}$ and the complexity of maintaining $\beta$ is also $\bO{KW}$. Thus, the iterations of this algorithm have the same complexity as RCD but are more efficient.

The distributed algorithm DICOD is faster compared to all the other sequential algorithms and the speed up increases with the number of cores. Also, DICOD has a shorter initialization time compared to the other algorithms. The first point in each curve indicates the time taken by the initialization. For all the other methods, the computations for constants -- necessary to accelerate the iterations -- have a computational cost equivalent to the on of the gradient evaluation. As the segments of signal in DICOD are smaller, the initialization time is also reduced. This shows that the overhead of starting the cores is balanced by the reduction of the initial computation for long signals. For shorter signals, we have observed that the initialization time is of the same order as the other methods. The spawning overhead is indeed constant whereas the constants are cheaper to compute for small signals.

\begin{figure}[t]
	\centering
	\includegraphics[width=.49\textwidth]{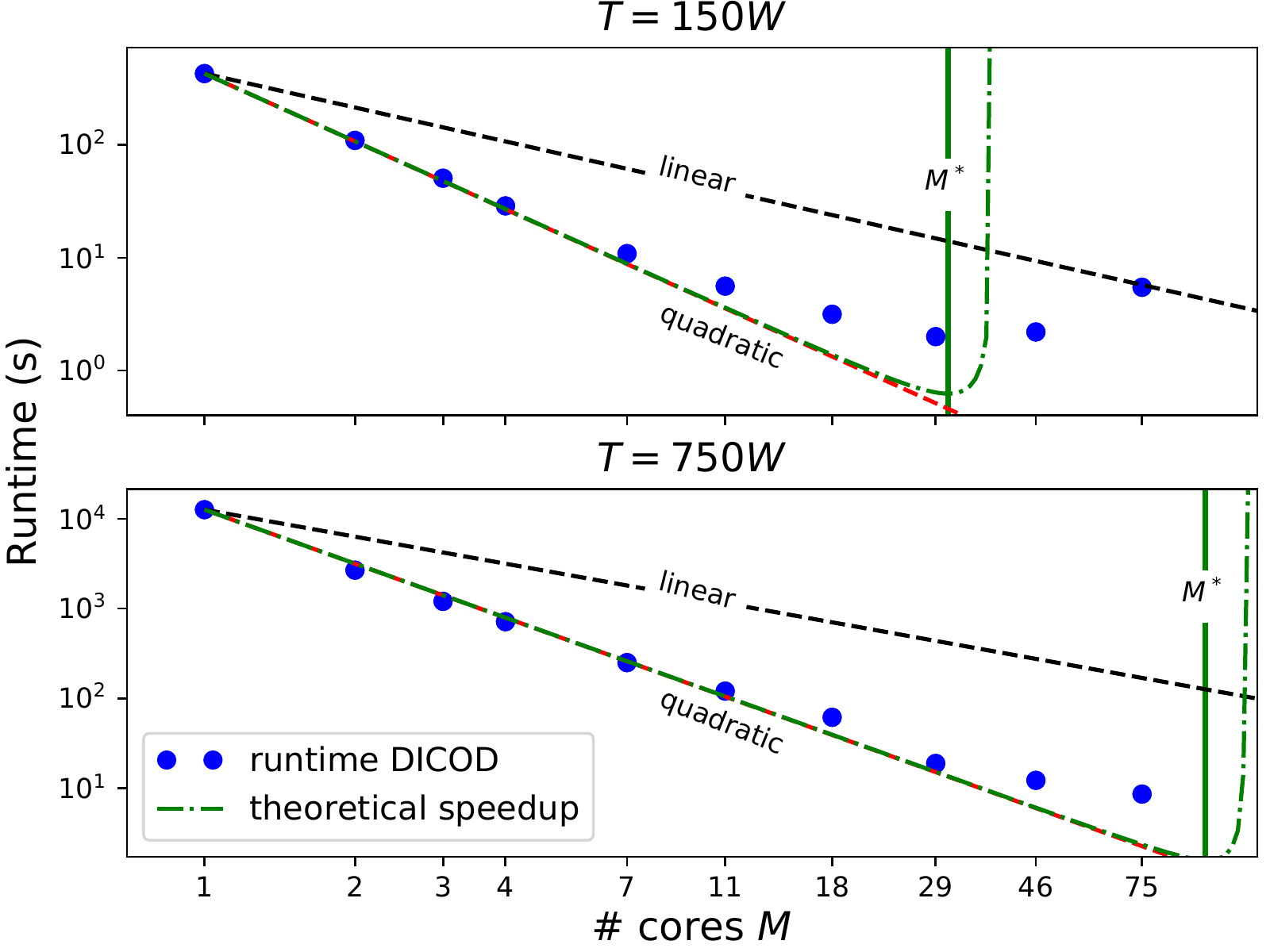}
	\caption{
		Speedup of DICOD as a function of the number of processes
		used, average over 10 run on different generated signals.
		This highlights a sharp transition between a regime of
		quadratic speedups and the regime where the interference
		are slowing down drastically the convergence.
	}
	\label{fig:scaling}
\end{figure}

%\subsection{Numerical Complexity}
%\label{sub:compl}

% subsection  (end)

\paragraph{Speedup Evaluation.}

\autoref{fig:scaling} displays the speedup of DICOD as a function of the number of cores. We used 10 generated problems for 2 signal lengths $T=150\cdot W$ and $T = 750\cdot W$ with $W=200$ and we solved them using DICOD$_M$ with a number of cores $M$ ranging from 1 to 75. The blue dots display the average running time for a given number of workers. For both setups, the speedup is super-linear up to the point where $M\alpha = \frac{1}{2}$. For small $M$ the speedup is very close to quadratic and a sharp transition occurs as the number of cores grows. The vertical solid green line indicates the approximate position of the maximal speedup given in \autoref{corr:speedup} and the dashed lined is the expected theoretical run time derived from the same expression. The transition after the maximum is very sharp. This approximation of the speedup for small values of $M\alpha$ is close to the experimental speedup observed with DICOD. The computed optimal value of $M^*$ is close to the optimal number of cores in these two examples.

% subsection res (end)
%%%%%%%%%%%%%%%%%%%%%%%%%%%%%%%%%%%%%%%%%%%%%%%%%%%%%%%%%%%%%%%%%%%%%%%%%%%%%%%%%%%%%%%%%%%%%%%%%%%%%%%%%%%%%%%%%%%%%%%%%%%%%%%%%%%%%%%%%%%%%%%%%%%%%%%%%%%%%%%%%%%%%%%%%%%%%%%%%%%%%%%%%%%%%%%%%%%%%%%%%%%%%%%%%%%%%%%%%%%%%%%%%%%%%%

% subsection stuff (end)

\crossref{}

\section{Conclusion}

In this work, we introduced an asynchronous distributed algorithm that is able to speed up the resolution of the Convolutional Sparse Coding problem for long signals. This algorithm is guaranteed to converge to the optimal solution of \autoref{eq:sparse_code} and scales superlinearly with the number of cores used to distribute it. These claims are supported by numerical experiments highlighting the performances of DICOD compared to other state-of-the-art methods. Our proofs rely extensively on the use of one dimensional convolutions. In this setting, a process $m$ only has two neighbors $m-1$ and $m+1$. This ensures that there is no high order interferences between the updates. Our analysis does not apply straightforwardly to distributed computation using square patches of images as the higher order interferences are more complicated to handle. A way to apply our algorithm with these guarantees to images is to split the signals along only one direction, to avoid higher order interferences. The extension of our results to this case is an interesting direction for future work.

\bibliographystyle{\Bibstyle}
\bibliography{../../../library}

\appendix

\setcounter{theorem}{0}
\setcounter{hypothesis}{0}
\setcounter{equation}{8}

\onecolumn

\icmltitle{Supplemetary materials -- DICOD: Distributed Convolutional Coordinate Descent for Convolutional  Sparse Coding}

\appendix

% this must go after the closing bracket ] following \twocolumn[ ...

 \section{Computation for the cost updates}
\label{sec:annex:cost_update}

When a coefficient $Z_k[t]$ is updated  to $u$, the cost update is a simple function of $Z_k[t]$ and $u$.

% section  (end)
 \begin{propositionA}
 	The update of the weight in $(k_0, t_0)$ from the value $Z_{k_0}[t_0]$ in $Z$ to $u \in \Rset$ in $Z^{(1)}$ gives a cost cost variation:
 	\begin{align*}
	  \smeq e_{k_0, t_0}(u) & = E(Z) - E(Z^{(1)})\\
	  					    & = \frac{\|\pmb D_{k_0}\|_2^2}{2} (Z_{k_0}[t_0]^{2} - u^2)
	  						- \beta_{k_0}[t_0](Z_{k_0}[t_0] - u)
	  						+ \lambda(|Z_{k_0}[t_0]| - |u|).
    \end{align*}
 \label{propA:pb_coord}
 \end{propositionA}%
%
%%%Proof 
 \begin{proof}
 Let $\alpha_{k_0}[t] = (X - \sum_{k=1}^KZ_k*\pmb D_k)[t] + \pmb D_{k_0}[t-t_0] Z_{k_0}[t_0]$ for all $t \in \llbracket 0..T-1\rrbracket$ and \\
        $Z^{(1)}_{k}[t] =  \begin{cases}
 		    u, &\text{ if } (k, t) = (k_0, t_0)\\
            Z_{k}[t], &\text{ elsewhere }\\
 		\end{cases}$ .
 	\begin{align*}
	   e_{k_0, t_0}(u) =& \frac{1}{2}\sum_{t=0}^{T-1} \left(X - \sum_{k=1}^KZ_{k}*\pmb D_{k}\right)^2[t] + \lambda \sum_{k=1}^K\|Z_{k}\|_1
	   	- \frac{1}{2}\sum_{t=0}^{T-1} \left(X - \sum_{k=1}^K Z^{(1)}_{k}*\pmb D_{k}\right)^2[t] + \lambda \sum_{k=1}^K\|Z^{(1)}_{k}\|_1 \\
	   =& \frac{1}{2}\sum_{t=0}^{T-1}\left(\alpha_{k_0}[t] - \pmb D_{k_0}[t-t_0] Z_{k_0}[t_0]\right)^2
	   	- \frac{1}{2}\sum_{t=0}^{T-1} \left(\alpha_{k_0}[t] - \pmb D_{k_0}[t-t_0] u\right)^2
	   	+ \lambda (|Z_{k_0}[t_0]| - |u|)\\
	   =& \frac{1}{2}\sum_{t=0}^{T-1} \pmb D_{k_0}[t-t_0]^2(Z_{k_0}[t_0]^2 - u^2)
	    - \sum_{t=0}^{T-1}\alpha_{k_0}[t]\pmb D_{k_0}[t-t_0] (Z_{k_0}[t_0] - u)
	   	+ \lambda (|Z_{k_0}[t_0]| - |u|)\\
	   =& \frac{\|\pmb D_{k_0}\|_2^2}{2} (Z_{k_0}[t_0]^2 - u^2)
	    - \underbrace{(\widetilde {\pmb D_{k_0}} * \alpha_{k_0})[t]}_{\beta_{k_0}[t_0]}(Z_{k_0}[t_0] - u)
	    + \lambda (|Z_{k_0}[t_0]| - |u|)
	\end{align*}%
	This conclude our proof.
\end{proof}
Using this result, we can derive the optimal value $Z'_{k_0}[t_0]$ to update the coefficient $(k_0, t_0)$ as the solution of the following optimization problem:
\begin{equation}
	\label{eqA:pb_coord}
	Z'_{k_0}[t_0] = \arg\min_{y\in \Rset} e_{k_0, t_0}(u) \sim \arg\min_{u\in \Rset}\frac{\| \pmb D_{k_0}\|_2^2}{2}\left(u - \beta_{k_0}[t_0]\right)^2 + \lambda |u|~.\\
\end{equation} 
In the case where two coefficients $(k_0, t_0), (k_1, t_1)$ are updated in the same iteration
to values $u$ and $Z'_{k_1}[t_1]$, we obtain the following cost variation.

\begin{propositionA}\label{prop:ii}
	The update of the weight $Z_{k_0}[t_0]$ and $Z_{k_1}[t_1]$ to values $Z'_{k_0}[t_0]$ and $Z'_{k_1}[t_1]$ 
	with $\Delta Z_k[t] = Z_k[t] - Z'_k[t]$ gives an update of the cost:
	\begin{align*}
		\Delta E_{k_0 k_1}[t_0, t_1] = \Delta E_{k_0}[t_0] + \Delta E_{k_1}[t_1]
		- \mathcal S_{k_0, k_1}[t_0-t_1]\Delta Z_{k_0}[t_0] \Delta Z_{k_1}[t_1]
	\end{align*}
\end{propositionA}

\begin{proof}
We define $Z_k^{(1)}[t] = \begin{cases}
	Z_{k_0}[t_0], &\text{ if } (k, t) = (k_0, t_0)\\
	Z_{k_1}[t_1], &\text{ if } (k, t) = (k_1, t_1)\\
	Z_k[t], &\text{ otherwise }
\end{cases}~.$\\[.3em]
Let $\alpha[t] = (X - \sum_{k=1}^KZ_{k}\pmb D_{k})[t] + \pmb D_{k_0}[t-t_0] Z_{k_0}[t_0] + \pmb D_{k_1}[t-t_1] Z_{k_1}[t_1]$.\\[.5em]
We have  $ \alpha[t] = \alpha_{k_0}[t] + \pmb D_{k_1}[t-t_1] Z_{k_1}[t_1] =  \alpha_{k_1}[t] + \pmb D_{k_0}[t-t_0] Z_{k_0}[t_0]$.
 	\begin{align*}
	   \Delta E_{k_0k_1}[t_0,t_1] =
	   				& \frac{1}{2}\sum_{t=0}^{T-1} \left(X - \sum_{k=1}^K Z_k*\pmb D_k\right)[t]^2 + \frac{1}{2}\sum_{k=1}^K\lambda \|Z_k\|_1 
	   				 - \sum_{t=0}^{T-1} \left(X - \sum_{k=1}^K Z_k^{(1)}*\pmb D_k\right)^2[t] + \lambda \sum_{k=1}^K\|Z_k^{(1)}\|_1 \\
	   =
	   				& \frac{1}{2}\sum_{t=0}^{T-1}\left(\alpha[t] - \pmb D_{k_0}[t-t_0] Z_{k_0}[t_0] - \pmb D_{k_1}[t-t_1] Z_{k_1}[t_1]\right)^2
	   				 + \lambda (|Z_{k_0}[t_0]| - |Z'_{k_0}[t_0]|)\\
	   				&- \frac{1}{2}\sum_{t=0}^{T-1} \left(\alpha[t] - \pmb D_{k_0}[t-t_0] Z'_{k_0}[t_0] - \pmb D_{k_1}[t-t_1] Z'_{k_1}[t_1]\right)^2 
	   				 + \lambda (|Z_{k_1}[t_1]| - |Z'_{k_1}[t_1]|)\\
	   =
	   				&\frac{1}{2}\sum_{t=0}^{T-1}\Biggl[\pmb D_{k_0}[t-t_0]^2(Z_{k_0}[t_0]^2-{Z'_{k_0}[t_0]}^2) 
	   				 + \pmb D_{k_1}[t-t_1]^2(Z_{k_1}[t_1]^2 - {Z'_{k_1}[t_1]}^2) \Biggr]\\
	  				& -\sum_{t=0}^{T-1}\Biggl[\alpha_{k_0}[t] \pmb D_{k_0}[t-t_0]\Delta Z_{k_0}[t_0] + \alpha_{k_1}[t_1] \pmb D_{k_1}[t-t]\Delta Z_{k_1}[t_1]
	  				\\ 
	  				&~~~~~~~~~~~~~~ +  \pmb D_{k_0}[t-t_0] \pmb D_{k_1}[t-t_1](\Delta Z_{k_0}[t_0] Z'_{k_1}[t_1] + \Delta Z_{k_1}[t_1] Z'_{k_0}[t_0])\\
	   				&~~~~~~~~~~~~~~ - \pmb D_{k_0}[t-t_0] \pmb D_{k_1}[t-t_1](Z_{k_0}[t_0]Z_{k_1}[t_1] - Z'_{k_0}[t_0]Z'_{k_1}[t_1])\Biggr]\\
	   				&~ + \lambda (|Z_{k_0}[t_0]| - |Z'_{k_0}[t_0]| + |Z_{k_1}[t_1]| - |Z'_{k_1}[t_1]|)  \\
	   =
	   				& \Delta E_{k_0}[t_0] + \Delta E_{k_1}[t_1] \\
					& 
					- \sum_{t=0}^{T-1}  \pmb D_{k_0}[t-t_0] \pmb D_{k_1}[t-t_1]\Bigl[
						Z_{k_0}[t_0]Z_{k_1}[t_1] - Z'_{k_0}[t_0]Z_{k_1}[t_1] - Z_{k_0}[t_0]Z'_{k_1}[t_1] +Z'_{k_1}[t_1]Z'_{k_0}[t_0] \Bigr]\\
	   =
	   				& \Delta E_{k_0}[t_0] + \Delta E_{k_1}[t_1] 
	   				 - \sum_{t=0}^{T-1} \pmb D_{k_0}[t] \pmb D_{k_1}[t+t_0-t_1](Z_{k_0}[t_0] - Z'_{k_0}[t_0])(Z_{k_1}[t_1] - Z'_{k_1}[t_1]) \\
	   =
	   				& \Delta E_{k_0}[t_0] + \Delta E_{k_1}[t_1] 
	   				 -  \widetilde{\pmb D_{k_0}}*\pmb D_{k_1}[t_0-t_1]\Delta Z_{k_0}[t_0] \Delta Z_{k_1}[t_1]
	\end{align*}
	By definition of $\mathcal S_{k_0, k_1}[t] = \widetilde{\pmb D_{k_0}}*\pmb D_{k_1}[t]$. This conclude our proof. \end{proof}

\section{Intermediate results}
\label{sec:intermed}

Consider solving a convex problem of the form:
\begin{equation}\label{eq:gen}
\min_{} E(Z) = F(Z) + \sum_{t=0}^{L-1}\sum_{k=1}^K g_i(Z_k[t])
\end{equation}
where F is differentiable and convex, and $g_i$ is convex. Let us first recall a theorem stated and proved in \cite{Osher2009}.

\begin{theoremA}\label{thm:opt_coord}
Suppose $F(z)$ is smooth and convex, with $\left|\frac{\partial^2 F}{\partial u_i \partial u_j}\right|_\infty \le M$, and $E$ is strictly convex with respect to any one variable $Z_i$, then the statement that $u = (u_1, u_2, \dots u_n)$ is an optimal solution of \mbox{\autoref{eq:gen}} is equivalent to the statement that every component $u_i$ is an optimal solution of $E$ with respect to the variable $u_i$ for any $i$.
\end{theoremA}

In the convolutional sparse coding problem, the function $F(Z) = \frac{1}{2}\|X - \sum_{k=1}^K Z_k*\pmb D_k\|$ is smooth and convex and its Hessian is constant.
The following \autoref{lem:scvx}, can be used to show that the function $E$ restricted to one of its variable is strictly convex and thus satisfies the condition of \mbox{\autoref{thm:opt_coord}}.

\begin{lemmA}\label{lem:scvx}
The function $f : \mathbb R \to \mathbb R$ defined for $\alpha, \lambda > 0$ and $b \in \mathbb R$ by $f(x) = \frac{\alpha}{2} (x- b)^2 + \lambda |x|$ is  $\alpha$-strongly convex.
\end{lemmA}

 \begin{proof}
The property of monotone subdifferential states that a function $f$ is $\alpha$-strongly convex if and only if
$$
\forall (x, x'), \hspace{.5cm} \langle f(x)-f(x'), x-x'\rangle \ge \alpha \|x-x'\|_2^2
$$
Let us define the subdifferential of $f$:
$$
	\partial f = \begin{cases}
		  \alpha(x - b) + \lambda \text{sign}(x) & \text{if } x \neq 0\\
		 -\alpha b + \lambda t, \text{ for } t\in \left[-1, 1\right] & \text{if } x = 0\\
				\end{cases}
$$ The inequality is an equality for $x = x'$.\\
If $x' = 0$, we get for $|t| \le 1$:
\begin{align*}
	\langle \alpha(x -b) + \lambda\text{sign}(x) + \alpha b- \lambda t), x\rangle
		&= \alpha x^2 + \lambda\underbrace{(|x| - tx)}_{\ge 0} \ge \alpha x^2 = \alpha(x-x')^2
\end{align*}
If $x' \neq 0$, we get:
\begin{align*}
	\langle \alpha (x - x') + \lambda(\text{sign}(x) - \text{sign}(x'), x - x'\rangle & = \alpha (x-x')^2 + \lambda(|x| + |x'| - \text{sign}(x)x' -\text{sign}(x')x)\\
    &= \alpha (x-x')^2 + \lambda(\underbrace{1 - \text{sign}(x)\text{sign}(x'))}_{\ge 0}(|x|+|x'|)\\
	&\ge \alpha (x-x')^2
\end{align*}
Thus $f$ is $\alpha$-strongly convex.
\end{proof}
This can be applied to the function $e_k,t$ defined in \autoref{eqA:pb_coord}, showing that the problem in one coordinate $(k, t)$ is $\|\pmb D_k\|_2^2$-strongly convex.

\section{Proof of convergence for DICOD (\autoref{thm:conv_dcp})}
\label{sec:annex:cvg}

\subsection{Lower bound on interference}
\label{sub:annex:cvg:interf}

We define
\begin{equation}
	C_{k_0,k_1}[t] = \frac{\mathcal S_{k_0, k_1}[t]}{\|\pmb D_{k_0}\|_2\|\pmb D_{k_1}\|_2}
\end{equation}
Let us first show how $C_{k_0,k_1}$ controls the interfering cost update.

\begin{proposition}
\label{prop:interf}
In case of concurrent update for coefficients $(k_0, t_0)$ and $(k_1, t_1)$, the cost update $\Delta E_{k_0k_1}[t_0,t_1]$ is bounded as
\begin{align*}
	\Delta E_{k_0k_1}[t_0,t_1] \ge 
		& \Delta E_{k_0}[t_0] + \Delta E_{k_1}[t_1]
		- 2 C_{k_0k_1}[t_0-t_1]\sqrt{\Delta E_{k_0}[t_0]\Delta E_{k_1}[t_1]}.
\end{align*}
\end{proposition}

\begin{proof}
The problem in one coordinate $(k, t)$ given all the other can be reduced \autoref{eqA:pb_coord}.
	Simple computations show that:
	\begin{equation}
	  \Delta E_k[t] = e_{k, t}(Z_k[t]) - e_{k, t}(Z'_{k}[t]).
	\end{equation}
	We have shown in \autoref{lem:scvx} that $e_{k, t}$ is $\| \pmb D_k\|^2_2$-Strong convex.
	Thus by definition of the strong convexity, and using the fact that  $Z'_{k}[t]$ is optimal for $e_{k, t}$
	\begin{equation}
		|e_{k, t}(Z_k[t]) - e_{k,t}(Z'_{k}[t])| \ge \frac{\| \pmb D_k\|_2^2}{2}(Z_k[t]-Z'_{k}[t])^2 
	\end{equation}
	\emph{i.e.}, $|\Delta Z_k[t]| \le \frac{\sqrt{2 \Delta E_k[t]}}{\| \pmb D_k\|_2}$,
	and the result is obtained using this inequality with \autoref{prop:ii}.
\end{proof}

% subsection annex:cvg:interf (end)

\subsection{Proof of \autoref{thm:conv_dcp}}
\label{sub:annex:cvg:proof}

% subsection  (end)

\begin{theorem}\label{thmA:conv_dcp}
If the following hypothesis are verified
\begin{hypothesis}\label{hyp:H1}
	For all $\smeq(k_0, t_0), (k_1, t_1)$ such
	that $\smeq t_0 \neq t_1 $,
	\[
		|C_{k_0k_1}[t_0-t_1]| < 1~.
	\]
\end{hypothesis}
\begin{hypothesis}\label{hyp:H2}
	There exists $A \in \mathbb N^*$ such that all cores
	$m \in [M]$ are updated at least once between iteration
	$q$ and $q+A$ if the solution is not locally optimal, 
	\emph{i.e.} $\Delta Z_k[t] = 0$ for all $(k, t) \in \mathcal C_m$  
\end{hypothesis}
\begin{hypothesis}\label{hyp:H3}
	The delay in communication between the processes is inferior to the update time.
\end{hypothesis}
Then, the DICOD algorithm using the greedy updates
$(k_0, t_0) = \arg\max_{(k, t) \in \mathcal C_m} |\Delta Z_k[t]|$
converges to the optimal solution $Z^*$ of (2).
\end{theorem}

\begin{proof}

If several updates $(k_0, t_0), (k_1, t_1), \dots (k_m, t_m)$
are updated in parallel without interference, then the update is
equivalent to the sequential updates of each $(k_q, t_q)$.
We thus consider that for each step $i$, without loss of generality
that
$$
	\Delta E^{(i)} = \begin{cases}
		\Delta E_{k_0}^{(i)}[t_0], &\text{ if  there is no interference}\\
		\Delta E_{k_0k_1}^{(i)}[t_0,t_1], &\text{ otherwise}
	\end{cases}
$$
If $\forall (k, t), \Delta Z_k^{(i)}[t] = 0$, then $Z^{(i)}$ is
coordinate wise optimal. Using the result from \autoref{thm:opt_coord},
$Z^{(i)}$ is optimal. Thus if $Z^{(i)}$ is not optimal,
$\Delta E_{k_0}^{(i)}[t_0] > 0$.

Using \autoref{prop:interf} and \autoref{hyp:H1} 
\[\Delta E_{k_0k_1}^{(i)}[t_0,t_1] > \left(\sqrt{\Delta E_{k_0}^{(i)}[t_0]} - \sqrt{\Delta E_{k_1}^{(i)}[t_1]}\right)^2 \ge 0~,\]
so the update $\Delta E^{(i)}$ is positive.

The sequence $(E(Z^{(i)}))_n$ is decreasing and bounded by 0.
It converges to $E^*$ and $\Delta E^{(i)} {\xrightarrow[n
 \to \infty]{}} 0$. As $\lim_{\|z\|_\infty \to \infty}
 E(z) = +\infty$, there exist $M\ge 0, i_0 \ge 0$ such that
 $\|Z^{(i)}\|_\infty \le M$ for all $i > i_0$. Thus, there exist a
 subsequence $(Z^{i_q})_q$ such that $Z^{i_q}
 \xrightarrow[q\to \infty]{} \bar z$.
By continuity of $E$, $E^* =  E(\bar z)$

Then, we show that $Z^{(i)}$ converges to a point $\bar z$
such that each coordinate is optimal for the one coordinate
problem.  By \autoref{prop:interf}, the sequence $(Z^{(i)})_i$
is $\ell_\infty$-bounded. It admits at least a limit point
$Z^{(i_q)} \xrightarrow[q \to\infty]{}\bar z$. Moreover,
the sequence $Z^{(i)}$ is a Cauchy sequence for the norm
$\ell_\infty$ as for $p, q > 0$
\begin{align*}
	\|Z^{(p)} - Z^{(q)}\|_\infty^2
		& \le \frac{2}{\| D\|_{\infty,2}^2}
			\sum_{l > q} \Delta E^{(l)}\\
		& = \frac{2}{\| D\|_{\infty,2}^2}
			\left( E(Z^{q}) - E^* \right)
			\underset{q \to \infty}{\to} 0
\end{align*}
Thus $Z^{(i)}$ converges to $\bar z$.

Let $m$ denote one of the $M$ cores and $(k, t)$ be coordinates
in $\mathcal C_m$. We consider the function
$h_{k, t}:\mathbb R^{K\times L} \to \mathbb R$ such that 
\[
	h(z) = Z'_{k}[t] = \frac{1}{\|\pmb D_k\|_2^2}
		\text{Sh}(\beta_k[t], \lambda)~.
\]
We recall that
\[
	\beta_{k}[t](Z) = \left(\widetilde{\pmb D_{k}}*\left(X-
		\sum_{\substack{k'=1\\k'\neq k}}^KZ_{k'}*\pmb D_{k'} - \Phi_{t}\left(Z_{k}\right)*\pmb D_{k}\right)\right)[t]
\]
The function $\phi: Z \to \beta_k[t](Z)$ is linear. As Sh is
continuous in its first coordinate and
$h(Z) = \text{Sh}(\phi(Z), \lambda)$,
the function $h_{k,t}$ is continuous.
For $(k, t) \in \mathcal C_m$,  the gap between $\bar Z_k[t]$ and $\bar Z'_{k}[t]$ is such that
\begin{align}
	|\bar Z_k[t] -\bar Z'_{k}[t]| &= |\bar Z_k[t] - h_{k, t}(\bar Z_k[t])| \nonumber\\
			&= \lim_{i\to\infty} |Z_k^{(i)}[t] - h(Z_k^{(i)}[t])|\nonumber\\
			&= \lim_{i\to\infty} |Z_k^{(i)}[t] - {y}_k^{(i)}[t]| \label{eq:kt}
\end{align}
Using \autoref{hyp:H2}, for all $i \in \mathbb N$, if $Z^{(i)}_k[t]$ is not
optimal, there exist $q_i \in [i, i+A]$ such that the
updated coefficient at iteration $q_i$ is $(k_{q_i}, t_{q_i}) \in \mathcal C_m$. As no update
are done on $\mathcal C_m$ coefficients between the updates $i$ and $q_i$,
$Z_{k}^{(i)}[t] = Z_k^{(q_i)}[t]$. By definition of the update,
\begin{align*}
	\left|Z_{k}^{(i)}[t] - {y}_{k}^{(i)}[t]\right| 
	 	& = \left|Z_{k}^{(q_i)}[t] - {y}_{k}^{(q_i)}[t]\right| \\
		& \le \left|Z_{k_{q_i}}^{(q_i)}[t_{q_i}] - {y}_{k_{q_i}}^{(q_i)}[t_{q_i}]\right|
				\tag*{\small (greedy updates)}\\
		& \le \frac{\sqrt{2 \Delta E^{(q_i)}}}{\|\pmb D_{k_{q_i}}\|} \underset{i\to \infty}{\to} 0
				\tag*{\small (\autoref{prop:interf})}\\
\end{align*}
Using this results with \autoref{eq:kt},
$\left|\bar Z_{k}[t] - {\bar y}_{k}[t]\right|=0$.
This proves that $\bar z$ is optimal in each coordinate. 
By \autoref{thm:opt_coord}, the limit point $\bar z$
is optimal for the problem (2).

\end{proof}

\section{Proof of DICOD speedup (\autoref{thm:speedup})}
\label{sec:annex:speedup}

\restatethm{\SpeedDICOD}{thm:speedup}
\begin{theorem}\label{thmA:speedup}
Let $\alpha = \frac{W}{T}$ and $M\in\mathbb N^*~.$ If $\alpha M < \frac{1}{4}$ and if
the non zero coefficients of the sparse code are distributed uniformly in time,
the expected speedup $\mathbb E [S_{dicod}(M)]$ is lower bounded by
\[
	\mathbb E[S_{dicod}(M)] \ge M^2 ( 1- 2\alpha^2M^2\left(\smeq1 +2\alpha^2M^2\right)^{\frac{M}{2}-1})~.
\]
\end{theorem}

This result can be simplified when the interference probability $(\alpha M)^2$  is small.

\restatethm{\corSpeedDICOD}{corr:speedup}
\begin{corollary}\label{corrA:speedup}
Under the same hypothesis,
the expected speedup $\mathbb E [S_{dicod}(M)]$ when $(M\alpha)^2 \to 0$ is
\[
\mathbb E[S_{dicod}(M)] \underset{\alpha \to 0}{\gtrsim} M^2(1-2\alpha^2M^2 + \mathcal O(\alpha^4M^4))~.
\]
\end{corollary}

\begin{proof}
There are two aspects involved in DICOD speedup:
the computational complexity and the acceleration due to
the parallel updates.

As stated in \autoref{sec:analysis}, the complexity of each
iteration for CD is linear with the length of the input signal
$T$. The dominant operation is the one that find the maximal coordinate.
In DICOD, each core runs the same iterations on a segment of size $\frac{T}{M}$.
The hypothesis $\alpha M < \frac{1}{4}$ ensures that the dominant operation is finding the maxima.
Thus, when CD run one iteration, one core of DICOD can run $M$ local iteration
as the complexity of each iteration is divided by $M$.

The other aspect of the acceleration is the parallel update of $Z$. 
All the cores perform their update simultaneously and
each update happening without interference can be considered
as a sequential update. Interfering updates do not degrade the cost. Thus, one 
iteration of DICOD with $N_i$ interference is equivalent to 
$M - 2*N_{interf}$ iterations using CD and thus,

\begin{equation}
	\label{eq:ndicod}
	\mathbb E[N_{dicod}] = M - 2*\mathbb E[N_{interf}]
\end{equation}

The probability of interference depends on the ratio between the length of the segments 
used for each cores and the size of the dictionary. If all the updates are spread uniformly
on each segment, the probability of interference between 2 neighboring cores is
$\left(\frac{MW}{T}\right)^2 = (M\alpha)^2$.

A process can only creates one interference with one of its neighbors. Thus,
an upper bound on the probability to get exactly $j \in [0, \frac{M}{2}]$ interferences is
\begin{equation*}
	\mathbb P(N_i = j) \le \binom{\frac{M}{2}}{j}(2\alpha^2M^2)^j
\end{equation*}
Using this result, we can upper bound the expected number of
interferences for the algorithm	
\begin{eqnarray*}
	\mathbb E[N_{interf}] = \sum_{j=1}^{\frac{M}{2}} j\mathbb P(N_{interf} = j)~,
		 & \le  &\sum_{j=1}^{\frac{M}{2}} j\binom{\frac{M}{2}}{j}(2\alpha^2M^2)^j,~\\
		& \le &\alpha^2M^3\left(1 +2\alpha^2M^2\right)^{\frac{M}{2}-1}~.
\end{eqnarray*}
Pluggin this result in \autoref{eq:ndicod} gives us:
\begin{equation}
	\label{eq:ldicod}
	\begin{split}
		\mathbb E[N_{dicod}] & \underset{\phantom{\alpha \to 0}}{\ge} \smeq M ( 1- 2\alpha^2M^2\left(\smeq1 +2\alpha^2M^2\right)^{\frac{M}{2}-1})~,\\
		&\smeq\underset{\alpha \to 0}{\gtrsim} M(1-2\alpha^2M^2 + \mathcal O(\alpha^4M^4))~.
	\end{split}
\end{equation}
Finally, by combining the two source of speedup, we obtain the desired result.
\[
	\mathbb E[S_{dicod}(M)] \ge M^2 ( 1- 2\alpha^2M^2\left(\smeq1 +2\alpha^2M^2\right)^{\frac{M}{2}-1})~.
\]

\end{proof}

% section  (end)

\end{document}